\documentclass[letterpaper]{article}
\usepackage{proceed2e}
\usepackage[margin=1in]{geometry}
\usepackage{times}

\usepackage{amsmath}
\usepackage{amsthm}
\usepackage{enumerate}
\usepackage{booktabs}
\usepackage{natbib}
\usepackage{pgf}
\usepackage{tikz}
\usepackage{float}
\usepackage{graphicx}
\usepackage{centernot}
\usepackage{mathtools}
\usepackage{subcaption}
\usepackage{epstopdf}
\usepackage[latin1]{inputenc}
\usepackage{multicol}
\usepackage{amssymb}
\usetikzlibrary{shapes.geometric}
\usepackage{pgfplots}
\usetikzlibrary{arrows,automata,positioning}
\usetikzlibrary{pgfplots.statistics}
\usepackage{caption}
\usepackage{filecontents}
\usepackage{bm}
\usepackage{enumitem}
\usepackage{hyperref}
\usepackage{adjustbox}
\usepackage{calc}

\hypersetup{
	colorlinks=false,
	bookmarksnumbered=false,
	pdftitle={Causal Discovery in the Presence of Measurement Error},
	pdfauthor={},
	pdfsubject={},
	pdfkeywords={},
	pdfcreator={},
	pdfproducer={},
	pdfpagemode=UseOutlines,
	pdfstartpage=1,
	pdfstartview={XYZ null null 1.00},
	pdfpagelayout=OneColumn,
	pdflang=English}

\usepgfplotslibrary{fillbetween}
\pgfplotsset{every tick label/.append style={font=\tiny}}
\pgfplotsset{compat=1.5}

\newcommand{\indep}{\mathrel{\perp\mspace{-10mu}\perp}}
\newcommand{\dep}{\centernot{\indep}}
\newcommand{\dsep}{\perp}
\newcommand{\dcon}{\not \perp}

\newcommand\B[1]{\bm{#1}}
\newcommand\C[1]{\mathcal{#1}}
\newcommand\given{\,|\,}

\DeclareMathOperator*{\argmax}{arg\,max}
\DeclareMathOperator*{\argmin}{arg\,min}

\newtheorem{assumptions}{Assumption}
\newtheorem{lemma}{Lemma}
\newtheorem{theorem}{Theorem}
\newtheorem{corollary}{Corollary}
\newtheorem{proposition}{Proposition}
\newtheorem{remark}{Remark}

\newlength{\emailslength}
\title{An Upper Bound for Random Measurement Error in Causal Discovery}
\author{
\textbf{Tineke Blom}\hspace{0.5pt}\textsuperscript{\normalfont{a,1}}
\hspace{3em}
\textbf{Anna Klimovskaia}\hspace{0.5pt}\textsuperscript{\normalfont{b,2}}
\hspace{3em}
\textbf{Sara Magliacane}\hspace{0.5pt}\textsuperscript{\normalfont{c,3}}
\hspace{3em}
\textbf{Joris M. Mooij}\hspace{0.5pt}\textsuperscript{\normalfont{a,4}}\\
\parbox{0.52\textwidth}{\vspace*{4pt}\raggedright
\textsuperscript{a}\hspace{0.5pt}{\small{}Informatics Institute, University of Amsterdam, The Netherlands}\\
\textsuperscript{b}\hspace{0.5pt}{\small{}Institute of Molecular Systems Biology, ETH Z{\"u}rich, Zwitserland}\\
\textsuperscript{c}\hspace{0.5pt}{\small{}IBM Research, Yorktown Heights, USA}}\\
\parbox{0.60\textwidth}{\vspace*{4pt}\raggedright\small{}E-mails: 
\textsuperscript{1}\hspace{0.5pt}\href{mailto:t.blom2@uva.nl}{\texttt{t.blom2@uva.nl}}, 
\textsuperscript{2}\hspace{0.5pt}\href{mailto:klimovskaia@imsb.biol.ethz.ch}{\texttt{klimovskaia@imsb.biol.ethz.ch}},\\\hspace{\emailslength}\textsuperscript{3}\hspace{0.5pt}\href{mailto:sara.magliacane@gmail.com}{\texttt{sara.magliacane@gmail.com}}, 
\textsuperscript{4}\hspace{0.5pt}\href{mailto:j.m.mooij@uva.nl}{\texttt{j.m.mooij@uva.nl}}}
}

\begin{document}
	
	\maketitle
	
	\begin{abstract}
		Causal discovery algorithms infer causal relations from data based on several assumptions, including notably the absence of measurement error. However, this assumption is most likely violated in practical applications, which may result in erroneous, irreproducible results. In this work we show how to obtain an upper bound for the variance of random measurement error from the covariance matrix of measured variables and how to use this upper bound as a correction for constraint-based causal discovery. We demonstrate a practical application of our approach on both simulated data and real-world protein signaling data.
	\end{abstract}
	
	\section{INTRODUCTION}
	
	The discovery of causal relations is a fundamental objective in science, and the interest in causal discovery algorithms has increased rapidly since they were first established in the 1990s \citep{Pearl2000,Spirtes2000}. In practice, it may happen that their predictions are not reproducible in independent experiments. In this article we show that the presence of measurement error may be a possible explanation for incorrect and inconsistent output and we propose a solution aimed to mitigate its ramifications.
	
	The presence of measurement error complicates causal discovery, because measured quantities are typically not causes of one another, even when the variables that they represent are. Consider the example in Figure \ref{fig:me_example}, and suppose that exercise $E$ is a variable that can be controlled in an experiment, weight loss $W$ can be measured very precisely, but the amount of burned calories $C$ cannot be observed directly. Suppose we do have a measured quantity $\tilde{C}=C+M_C$ with $M_C$ a measurement error. Even though exercise and weight loss are independent conditional on burned calories, they are not when we condition on the measurement $\tilde{C}$. If $M_C$ is large, one might even find that the measurements of the calories are independent of exercise conditional on the weight loss. A researcher who is unaware of the measurement error could then draw incorrect conclusions (e.g.\ weight loss causes the burning of calories).
	
	\begin{figure}[ht]
		\centering
		\begin{tikzpicture}[->,>=stealth,shorten >=1pt,auto,node distance=1.2cm,
		semithick,square/.style={regular polygon,regular polygon sides=4}]
		\node[state, fill=lightgray] (X) [] {$E$};
		\node[state] (Y) [right of=X] {$C$};
		\node[state, fill=lightgray] (Z) [right of=Y] {$W$};
		\node[state, fill=lightgray] (Ym) [below of=Y] {$\tilde{C}$};
		
		\path (Y) edge node {} (Z)
		(X) edge node {} (Y)
		(Y) edge node {} (Ym);
		\end{tikzpicture}
		\caption{Example of causal discovery in the presence of measurement error. Gray shaded nodes are observed variables, white nodes are latent variables.}
		\label{fig:me_example}
	\end{figure}
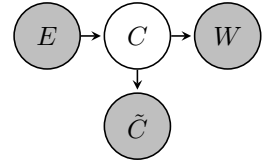
	
	The example in Figure \ref{fig:me_example} illustrates the crucial difference between measurement error and disturbance terms that are usually considered in causal models. In particular, the fluctuations that are due to measurement error do not propagate to effect variables (e.g.\ measurement noise $M_C$ in $\tilde{C}$ cannot be seen in $E$), whereas the effects of unmodeled causes do.
	
	Following \citep{Scheines2016, Zhang2017, Pearl} and \citep{Kuroki2014}, we focus on \emph{random} measurement error, an independent random variable that adds noise to the measurement of one variable in a model. We present a method that identifies an upper bound for the variance of random measurement error. This result builds on previous work where the identification of sets of variables that are d-separated by a common latent variable using vanishing tetrad constraints is considered, see \citep{Silva2001, Pearl, Bollen1989, Sullivant2010}. Uncertainty regarding the size of the measurement error can be propagated to an uncertainty in the partial correlations of the latent  variables that are yet unperturbed by measurement error, see also \citep{Harris2013a}. This uncertainty can then be taken into account when performing statistical tests so that we have outputs: dependent, independent, or unknown. Although these types of outputs for independence tests have been already used in previous work, e.g. \citep{Triantafillou2017}, in that case the thresholds for the different decisions were hyperparameters of the algorithm, while we provide an adaptive and more principled way to set them. Similarly to previous work, our approach relies on strong faithfulness \citep{Spirtes2000, Kalisch2007, Maathuis2010} but crucially it does not require causal sufficiency, i.e.\ the absence of unmeasured confounders, as \citet{Zhang2017} do.
	
	In this work, we propose a practical correction method for measurement error in the context of constraint-based causal discovery. We demonstrate the effectiveness of our approach in identifying causal structures using Local Causal Discovery (LCD) \citep{Cooper1997} both on simulated data and real-world protein signaling data. Although we focus on one particular causal discovery algorithm, our ideas can be applied to other constraint-based causal discovery algorithms as well, but we consider this to be outside of the scope of this paper.
	
	\section{PRELIMINARIES}
	\label{sec:preliminaries}
	
	For the remainder of this paper, variables will be denoted by capital letters and sets of variables by bold capital letters. We will assume that the data-generating processes described here can be modeled by a causal graph $\C{G}$ with nodes $\B{V}$ and directed and bidirected edges $\B{E}$, where some of the variables in $\B{V}$ may be latent. When there is a directed edge from a variable $X$ to a variable $Y$, we say that $X$ is a direct cause of $Y$. When there is a sequence of directed edges from $X$ to $Y$ with all arrowheads pointing towards $Y$ we call it a directed path, and we say that $X$ is an ancestor of $Y$. Bidirected edges between two variables $X$ and $Y$ are used to represent hidden confounders. Conditional independence between $\bm{X}$ and $\bm{Y}$ while controlling for variables in $\bm{Z}$ is denoted by $\bm{X}\indep \bm{Y} \given \bm{Z}$. If $\bm{Z}$ d-separates $\bm{X}$ from $\bm{Y}$, we denote this as $\bm{X} \dsep \bm{Y} \given \bm{Z}$.
	
	In the absence of measurement error, the following commonly made assumptions allow us to relate conditional (in)dependences between disjoint sets of variables $\bm{X},\bm{Y}$, and $\bm{Z}$ to d-separation in an underlying causal graph $\C{G}$ \citep{Pearl2000, Spirtes2000}. Throughout the remainder of this paper we will assume that the common assumptions hold. 
	 
	\begin{assumptions}[Common Assumptions]
		\label{ass:common}
		$\,$ 
		\begin{enumerate}[topsep=0pt, itemsep=0.1ex, partopsep=-5ex, parsep=1ex]
			\item There are no directed cycles in the causal graph.
			\item Causal Markov Property: For all disjoint sets of variables $\B{X},\B{Y},\B{Z}$:  $\B{X} \dsep \B{Y} \given \B{Z} \ \implies \B{X} \indep \B{Y} \given \B{Z}$.
			\item Causal Faithfulness: For all disjoint sets of variables $\B{X},\B{Y},\B{Z}$: $\B{X} \indep \B{Y} \given \B{Z} \implies \B{X} \dsep \B{Y} \given \B{Z}$.
			\item No selection bias is present.
		\end{enumerate}
	\end{assumptions}
	
	\paragraph{Local causal discovery}
	
	The LCD (Local Causal Discovery) algorithm is a straight-forward and efficient search method to detect one specific causal structure from experimental data using dependence relations between variables in $\bm{V}$ \citep{Cooper1997}.\footnote{\citet{Triantafillou2017} give an conservative variant of LCD with an application to protein signaling data.} LCD uses both (conditional) independences and background knowledge to recover causal relations from data. 
	
	The LCD algorithm looks for triples of variables $(X, Y, Z)$ for which (a) $X$ is not caused by any observed variable and (b) the following (in)dependences hold: $X\dep Y$, $Y\dep Z$, and $X\indep Z\mid Y$. We henceforth call such triples \emph{LCD triples}. Under the common assumptions, the causal model that corresponds to this independence pattern is shown in Figure \ref{fig:lcd triple}.
	
	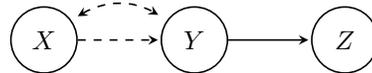
\begin{figure}[H]
		\centering
		\begin{tikzpicture}[->,>=stealth,shorten >=1pt,auto,node distance=2cm,
		semithick,square/.style={regular polygon,regular polygon sides=4}]
		\node[state] (X) [] {$X$};
		\node[state] (Y) [right of=X] {$Y$};
		\node[state] (Z) [right of=Y] {$Z$};	
		
		\path (Y) edge node {} (Z);
		\begin{scope}[dashed]
		\path (X) edge node { } (Y);
		\end{scope}
		\begin{scope}[dashed, style={<->}]
		\path ([yshift=2ex]X.east) edge[bend left] node { } ([yshift=2ex]Y.west);
		\end{scope}
		\end{tikzpicture}
		\caption{An LCD triple has the above causal structure, with at least one of the dashed arrows present.}
		\label{fig:lcd triple}
	\end{figure}
	
	\paragraph{Conditional (in)dependence testing}
	In practice, constraint-based causal discovery algorithms rely on a statistical test to assess the (in)dependence relationships between variables. For data that has a multivariate Gaussian distribution, a (conditional) independence corresponds to a vanishing (partial) correlation coefficient. For random variables $(X_1,\ldots,X_D)\sim\mathcal{N}(\boldsymbol{\mu},\Sigma)$, the Pearson partial correlation can be calculated from the inverse covariance matrix, which we will denote by $\Lambda=\Sigma^{-1}$.
	
	Conventionally, one calculates a p-value $p_T$ for the (conditional) dependence between variables, so that dependence relations can be determined by
	\begin{equation}
	\label{eq:thresh}
	\begin{cases}
	X\dep Y\mid \B{Z} & \,\text{ if }  p_T < \alpha\\
	X\indep Y\mid \B{Z} & \,\text{ if }  p_T > \beta,
	\end{cases}.
	\end{equation}
	where $\alpha$ and $\beta$ are thresholds for dependence and independence respectively. The nature of the relation is undecided when $\alpha\leq p_T\leq \beta$. Usually only a single threshold $\alpha=\beta=0.01$ or $\alpha=\beta=0.05$ is used.
	
	\section{CAUSAL DISCOVERY UNDER MEASUREMENT ERROR}
	\label{sec:me}
	In this section we illustrate some possible negative effects of random measurement error on constraint-based causal discovery. To that end, we analyze the behavior of partial correlations for increasing measurement error in a simple model. We consider \emph{random} measurement error, which is a vector of independent noise variables $\boldsymbol{M}=(M_1,\ldots,M_n)$. The measurements of the random vector $\boldsymbol{X}=(X_1,\ldots,X_n)$ are then given by $\tilde{\boldsymbol{X}}=(\tilde{X}_1,\ldots\tilde{X}_2) = \boldsymbol{X}+\boldsymbol{M}$. This means that a measurement node $\tilde{X}_i$ is always child-less and has precisely two parents: $X_i$ and the measurement error source $M_i$. 
	
	In many practical applications, it is reasonable to assume that the measurement noise has a Gaussian distribution. For instance, when the measurement noise is the sum of many small independent sources of error, the measurement error approximates a normal distribution because of the central limit theorem. In this article we consider the case where the measurement error is Gaussian so that the measurement noise variables are given by $\boldsymbol{M}=(M_1,\ldots,M_n)\sim\mathcal{N}(0,\Sigma_{\boldsymbol{M}})$, where $\Sigma_{\boldsymbol{M}}$ is a diagonal matrix.

	\subsection{MOTIVATIONAL EXAMPLE} We illustrate the effects of measurement error on the following structural causal model:
	\begin{align*}
		X_1 &= E_1\\
		X_2 &= \beta_{12} X_1 + E_2\\
		X_3 &= \beta_{23} X_2 + E_{3} \\
		\tilde{X}_2	&= X_2 +  M_{2}
	\end{align*}	
	where $X_1$ and $X_3$ are not affected by measurement error. In this model $E_1, E_2,$ and $E_3$ are normally distributed noise variables and $M_2$ is a normally distributed random measurement error. The observed variables are $X_1, \tilde{X}_2,$ and $X_3$, where the second represents the corrupted measurement of $X_2$. The corresponding causal graph is displayed in Figure \ref{fig:random_me_model}. 
	
	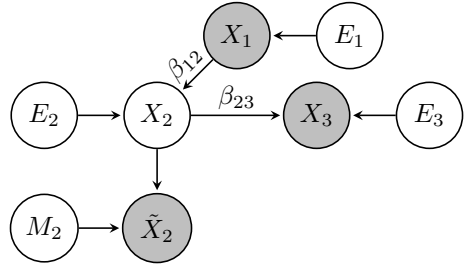
\begin{figure}
		\centering
		\begin{tikzpicture}[->,>=stealth,shorten >=1pt,auto,node distance=1.5cm,
		semithick,square/.style={regular polygon,regular polygon sides=4},el/.style = {inner sep=2pt, align=left, sloped},
		every label/.append style = {font=\tiny}]
		
		\node[state, fill=lightgray] (X1)              {$X_1$};
		\node[state] (X2) [below left of=X1] {$X_2$};
		\node[state, fill=lightgray] (X3) [below right of=X1] {$X_3$};
		\node[state, fill=lightgray]         (X2m) [below of=X2] {$\tilde{X}_2$};
		\node[state] (e1) [left of = X2] {$E_{2}$};
		\node[state] (e2) [right of = X3] {$E_{3}$};
		\node[state] (e3) [left of = X2m] {$M_{2}$};
		\node[state] (e5) [right of = X1] {$E_{1}$};

		\path 	(X1) edge node[el,above] {$\beta_{12}$} (X2)
		(X2) edge node[el,above] {$\beta_{23}$} (X3)
		(X2) edge node[above]  {} (X2m)
		(e1) edge node[el,below] {} (X2)
		(e2) edge node[el,below] {} (X3)
		(e3) edge node[el,above] {} (X2m)
		(e5) edge node[el,above] {} (X1);
		\end{tikzpicture}
		\caption{Causal graph of a model with random measurement error on $X_2$. Gray shaded nodes are observed variables (the others are latent), and coefficients alongside the arrows represent the coefficients in the model.}
		\label{fig:random_me_model}
		\vspace{-6pt}
	\end{figure}
	
	Note that in the random measurement error model, $(X_1,X_2,X_3)$ has the causal structure of an LCD triple, but $(X_1,\tilde{X}_2,X_3)$ does not. Therefore $X_1\indep X_3\given X_2$ and the partial correlation for the latent unmeasured variables satisfies $\rho_{13|2}=0$. Let $\tilde{\Sigma}$ be the covariance matrix of $(X_1,\tilde{X}_2, X_3)$ and $\tilde{\Lambda}$ its inverse. Then we have that
	\begin{align*}
	\tilde{\rho}_{13|2} = - \frac{\tilde{\Lambda}_{13}}{\sqrt{\tilde{\Lambda}_{11}\tilde{\Lambda}_{33}}} = \frac{-\beta_{12}\beta_{23}\tilde{\Sigma}_{11}\text{var}(M_2)}{|\tilde{\Sigma}|\sqrt{\tilde{\Lambda}_{11}\tilde{\Lambda}_{33}}} \neq 0,
	\end{align*}
	for non-zero parameters, so that $X_1\dep X_3\given \tilde{X}_2$.
		
	\begin{remark}
		A statistical test with conventional thresholds would conclude that $X_1$ and $X_3$ are conditionally dependent conditional on the measurement $\tilde{X}_2$, if the measurement error is large enough. If we would incorrectly assume that there is no measurement error, so that $X_2=\tilde{X}_2$, then the Markov assumption would appear to be violated.
	\end{remark}
	
	\subsection{EMPIRICAL STUDY}	
	For a better understanding of the impact of measurement error on causal discovery, we consider the effect of varying the measurement error variance $\text{var}(M_2)$ relative to the total variance of the measurement $\tilde{X}_2$ on the partial correlations in the motivational example.

	Figure \ref{fig:random measurement error} shows the effect of increasing relative random measurement error on different partial correlations, where the dotted lines represent the $\alpha=0.05$ threshold at different sample sizes. It can be seen that for zero measurement error (so that $\tilde{X}_2=X_2$), only the yellow line is below the red and black dotted lines. In that case a conventional statistical test would indicate that all variables are marginally dependent and $X_1\indep X_3\given\tilde{X}_2$, so that $(X_1,\tilde{X}_2,X_3)$ is an LCD triple, and the directed edge from $\tilde{X}_2$ to $X_3$ can be detected. For relative measurement errors larger than $\sim 0.25$ this conditional independence is no longer detected (because the yellow line is above the black-dotted line). 
		
	In Figure \ref{fig:random measurement error} we can also observe that for sample size $100$ and a relative measurement error larger than $\sim 0.3$, a conventional statistical test would indicate that $X_1\indep \tilde{X}_2\given X_3$ since the partial correlation $\tilde{\rho}_{12|3}\approx 0$ (i.e.\ below the red dotted line) and all other (partial) correlations indicate a dependence (i.e.\ above the red dotted line). Causal discovery algorithms cannot recover the correct causal structure from these constraints. In fact, the LCD algorithm would conclude that $(X_1,X_3,\tilde{X}_2)$ is an LCD triple so that there must be a directed edge in the \emph{reversed} direction.
	
	\begin{remark}
		The results of constraint-based causal discovery may depend on the sample size. This can be better understood by observing that the dependences that are identified by a statistical test, depend both on the size of the measurement error and the sample size. This may lead to inconsistent causal discoveries, which cannot be reproduced on new datasets.
	\end{remark}
	
	This example shows how measurement error interferes with detecting the correct causal structures, which may lead to edge deletions, insertions or reversals. Note that although we focused on the LCD algorithm here, the conclusions that we draw are more generally applicable to constraint-based causal discovery algorithms.
	
	\begin{remark}
		For relative measurement error of $\sim 0.25$ a conflicting set of (in)dependences arises for $n=100$. Since both the yellow and purple line are below the red dotted line, a statistical test would indicate that $X_1\indep X_3\given\tilde{X}_2$ and $X_1\indep\tilde{X}_2\given X_3$, while all variables are marginally dependent. But there is no model that satisfies the common assumptions and these (in)dependences.
	\end{remark}
	
	\definecolor{mycolor1}{rgb}{0.00000,0.75000,0.75000}
	\definecolor{mycolor2}{rgb}{0.75000,0.00000,0.75000}
	\definecolor{mycolor3}{rgb}{0.75000,0.75000,0.00000}
	\definecolor{ao(english)}{rgb}{0.11, 0.35, 0.02}
	\definecolor{applegreen}{rgb}{0.4, 0.69, 0.2}
	
	\begin{figure}
		\centering
		\begin{tikzpicture}
		\begin{axis}[
		width=0.7\linewidth,
		height=140pt,
		xlabel={$\text{var}(M_2) / \text{var}(\tilde{X}_2)$},
		xmin=0,xmax=1.0,xtick={0,0.2,0.4,0.6,0.8,1.0},
		ylabel={(partial) correlation},
		ymin=-0.1,ymax=1.0,ytick={0,0.5,1.0},
		legend pos=outer north east,
		legend columns=1,
		legend cell align={left}
		]
		
		\addplot[red!25!white, line width=0.5pt, forget plot] coordinates {
			(0,0.1975)
			(1.0,0.1975)
		};
		\addplot[black!25!white, line width=0.5pt, forget plot] coordinates {
			(0,0.0625)
			(1.0,0.0625)
		};
		
		\addplot[color=blue, line width=1.2pt] table[x index = {0}, y index = {1}, col sep = comma] {rand_me.csv};
		\addplot[color=black!50!green, line width=1.2pt] table[x index = {0}, y index = {2}, col sep = comma] {rand_me.csv};
		\addplot[color=mycolor1, line width=1.2pt] table[x index = {0}, y index = {3}, col sep = comma] {rand_me.csv};
		\addplot[color=mycolor2, line width=1.2pt] table[x index = {0}, y index = {4}, col sep = comma] {rand_me.csv};
		\addplot[color=mycolor3, line width=1.2pt] table[x index = {0}, y index = {5}, col sep = comma] {rand_me.csv};
		
		\addplot[red!25!white, line width=0.5pt, forget plot] coordinates {
			(0,-0.1975)
			(1.0,-0.1975)
		};
		\addplot[black!25!white, line width=0.5pt, forget plot] coordinates {
			(0,-0.0625)
			(1.0,-0.0625)
		};
		
		\addplot[dashed, red, line width=1.2pt] coordinates {
			(0,0.1975)
			(1.0,0.1975)
		};
		\addplot[dashed, black, line width=1.2pt] coordinates {
			(0,0.0625)
			(0.903614457831325,0.0625)
		};
	
		\addplot[dashed, red, line width=1.2pt] coordinates {
			(0,-0.1975)
			(1.0,-0.1975)
		};
		\addplot[dashed, black, line width=1.2pt] coordinates {
			(0,-0.0625)
			(1.0,-0.0625)
		};
		
		\legend{$\tilde{\rho}_{12}$, $\tilde{\rho}_{13}$, $\tilde{\rho}_{23}$, $\tilde{\rho}_{12 | 3}$, $\tilde{\rho}_{13 | 2}$, $n=100$, $n=1000\;$};
		\end{axis}
		\end{tikzpicture}
		\caption{Partial correlations in the random measurement error model in Figure \ref{fig:random_me_model}. The dotted lines represent the critical values for the correlation at a significance level of $\alpha=5\%$ for different sample sizes. The parameter settings were $\beta_{12}=0.6$, $\beta_{23}=1.2$ and all noise variables had variance $1.0$.}
		\label{fig:random measurement error}
		\vspace{-8pt}
	\end{figure}
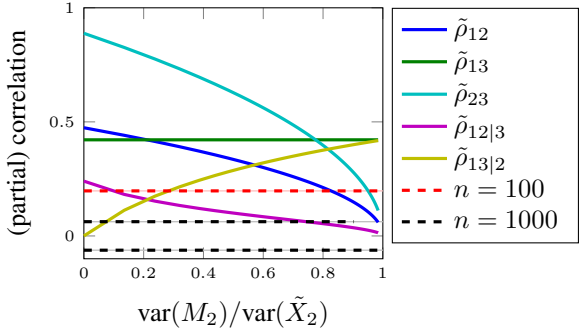

	\section{ERROR BOUND DETECTION}
	\label{sec:upperbound}
	Recall that the true covariances of $D$ random variables $\Sigma$, measurements $\tilde{\Sigma}$ and random measurement errors $\Sigma_{\boldsymbol{M}}$ are related as follows:
	\begin{equation*}
	\Sigma_{\boldsymbol{M}}=\tilde{\Sigma} - \Sigma = \text{diag}(m_1,\ldots,m_D)
	\end{equation*}
	where $m_1,\ldots,m_D>0$ are the variances of the random measurement error associated with each variable. In this section we show how, under certain conditions, an upper bound for the variance of random measurement error can be obtained from observational data with random measurement error. 
	
	\begin{remark}
		Given an (unbiased) estimate of $\Sigma_{\boldsymbol{M}}$, we can simply adjust the covariance matrix $\tilde{\Sigma}$ as suggested by \citet{Pearl}. In practice such an estimate of the covariance matrix of measurement error may not be available.
	\end{remark}
	
	We consider latent random variables $X_1,\ldots,X_4$ and their corresponding measurements $\tilde{X}_1,\ldots,\tilde{X}_4\in\boldsymbol{V}$ with true covariance matrices $\Sigma$ and $\tilde{\Sigma}$ respectively. Our upper bound result relies on Lemma \ref{lemma:dsep pattern} which is due to \citet{Silva2001} and gives conditions\footnote{These conditions are known as \emph{tetrad conditions} in the literature, see \citet{Bollen1989, Sullivant2010, Drton2006,Sullivant2010}} under which there exists a latent variable that d-separates the measured variables $\tilde{X}_1,\ldots,\tilde{X}_4$.

	\begin{lemma}
		\label{lemma:dsep pattern}
		Let $X_1,\ldots, X_4$ be variables in a linear-Gaussian model and let $\tilde{X}_1,\ldots,\tilde{X}_4$ be their measurements with random measurement error. If the correlations satisfy $\tilde{\rho}_{ij}\neq 0$ for all $i,j\in\{1,\ldots,4\}$ and $\tilde{\Sigma}_{12}\tilde{\Sigma}_{34}=\tilde{\Sigma}_{13}\tilde{\Sigma}_{24}=\tilde{\Sigma}_{14}\tilde{\Sigma}_{23}$, then there exists a node $L$ in the true underlying DAG such that $\tilde{X}_i\dsep \tilde{X}_j\given L$ for all $i\neq j\in\{1,\ldots,4\}$.
	\end{lemma}
	\begin{proof}
		The proof can be found in \citet{Silva2001}.
	\end{proof}
	
	When there exists a node $L$ that d-separates $\tilde{X}_1,\ldots,\tilde{X}_4$, then the causal graph and latent structure are represented by the causal graph in Figure \ref{fig:upperbound pattern}. This follows from the fact that the variables with measurement error $\tilde{X}_i$ can only have incoming arrows from $X_i$ and $M_i$ and never have any outgoing arrows. Because $L$ d-separates all $\tilde{X}_i$ there can be no collider at $L$.
		
	Before we present our upper bound result, we introduce an adjusted covariance matrix:
	\begin{equation*}
	\tilde{\Sigma}(u,j) = \tilde{\Sigma} - u\, \text{diag}(e_j),
	\end{equation*}
	where $j\in\{1,2,3,4\}$ and $e_j$ is a standard basis vector. For all $u$ such that $\tilde{\Sigma}(u,j)$ is a valid covariance matrix, the adjusted partial correlations $\tilde{\rho}^u_{ik|j}$ may be calculated from $\tilde{\Lambda}(u,j)=(\tilde{\Sigma}(u,j))^{-1}$ as follows:
	\begin{equation}
	\label{eq:adjusted partial correlation}
	\tilde{\rho}^u_{ik|j} = -\frac{(\tilde{\Lambda}(u,j))_{ik}}{\sqrt{(\tilde{\Lambda}(u,j))_{ii} (\tilde{\Lambda}(u,j))_{kk}}}.
	\end{equation}	
	
	\begin{figure}
		\centering
		\begin{tikzpicture}[->,>=stealth,shorten >=1pt,auto,node distance=1.2cm,
		semithick,square/.style={regular polygon,regular polygon sides=4},el/.style = {inner sep=2pt, align=left, sloped},
		every label/.append style = {font=\tiny}]
		
		\node[state, fill=lightgray] (X1m) {$\tilde{X}_1$};
		\node[state] (X1) [right of=X1m] {$X_1$};
		\node[state] (L) [below=0.6cm of X1] {$L$};
		\node[state] (X3) [below of=L] {$X_3$};
		\node[state] (X2) [left of= X3] {$X_2$};
		\node[state] (X4) [right of=X3] {$X_4$};
		\node[state, fill=lightgray] (X2m)[below of=X2] {$\tilde{X}_2$};
		\node[state, fill=lightgray] (X3m)[below of=X3] {$\tilde{X}_3$};
		\node[state, fill=lightgray] (X4m)[below of=X4] {$\tilde{X}_4$};
		
		\path 	(X1) edge node[el] {} (X1m);
		\path 	(X2) edge node[el] {} (X2m);
		\path 	(X3) edge node[el] {} (X3m);
		\path 	(X4) edge node[el] {} (X4m);
		\path 	(L) edge node[el] {} (X2);
		\path 	(L) edge node[el] {} (X3);
		\path 	(L) edge node[el] {} (X4);
		\path 	(X1) edge[bend left, dashed] node[el] {} (L);
		\path 	(L) edge[bend left, dashed] node[el] {} (X1);
		\path 	(L) edge[dashed, style={<->}] node {} (X1);
		
		\end{tikzpicture}
		\caption{Causal graph of upper bound pattern for model with random measurement error, where at least one of the dashed edges is present. The indexes $1,\ldots 4$ can be permuted. Noise variables $E_1,\ldots,E_4$ may be present but are not drawn. Measurement errors $M_1,\ldots,M_4$ are present but not drawn.}
		\label{fig:upperbound pattern}
		\vspace{-8pt}
	\end{figure}
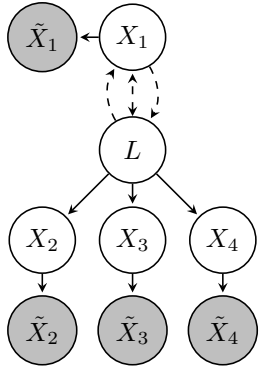
	
	Theorem \ref{thm:upperbound pattern} shows how the adjusted partial correlation is related to the underlying causal graph in Figure \ref{fig:upperbound pattern}. Corollary \ref{cor:measurement error bound} shows how we can use adjusted partial correlations to find an upper bound for the measurement error on one variable. 
	\newpage
	\begin{theorem}
		\label{thm:upperbound pattern}
		Let $X_1,\ldots, X_4$, $\tilde{X}_1,\ldots,\tilde{X}_4$ and $\tilde{\rho}_{ij}$ be as in Lemma \ref{lemma:dsep pattern}. The true underlying DAG is as in Figure \ref{fig:upperbound pattern} if and only if there exists $u>0$ such that $\tilde{\rho}^u_{13|2}=\tilde{\rho}^u_{14|2}=\rho^u_{34|2}=0$.
	\end{theorem}
	\begin{proof}
		The proof can be obtained by explicitly calculating the adjusted partial correlations and applying Lemma \ref{lemma:dsep pattern}. A complete proof can be found in the supplementary material.
	\end{proof}
	
	\begin{corollary}
		\label{cor:measurement error bound}
		Let $m_2$ be the variance of the measurement error on $X_2$. If $\tilde{\rho}^{u^*}_{13|2}=\tilde{\rho}^{u^*}_{14|2}=\tilde{\rho}^{u^*}_{34|2}=0$ for some ${u^*}>0$ then $m_2\leq {u^*}$.
	\end{corollary}
	\begin{proof}
		Follows from the proof of Theorem \ref{thm:upperbound pattern}.
	\end{proof}

	These results can also be applied in a practical, more general setting. If data is generated from a random measurement error model for variables $\boldsymbol{V}$, we consider subsets of four variables. If we can find an adjustment $u^*$ on the covariance matrix of this subset of variables so that the adjusted partial correlations in Theorem \ref{thm:upperbound pattern} vanish, then Corollary \ref{cor:measurement error bound} implicates that this adjustment is an upper bound for the variance of random measurement error. To ensure that the causal structure of these four variables is as in Figure \ref{fig:upperbound pattern}, we can test for the constraints in Lemma \ref{lemma:dsep pattern} (see \citep{Bollen1989, Silva2001, Thoemmes2018}). When all variables are measured in a similar manner, it may be reasonable to assume that the variance of the measurement error is the same for all variables. Under this assumption, the upper bound for the measurement error can be extended to an upper bound for the measurement error variance on all variables.
	
	\paragraph{Data simulations}
	To empirically test the performance of the upper bound, we simulated $10000$ datapoints for $10000$ random models with causal structures as in Figure \ref{fig:upperbound pattern} with parameters chosen uniformly from the interval $[-1,1]$ and error variances chosen uniformly from the interval $[0.5,1]$. We added random measurement error to each variable (the same variance was used for all variables) and minimized the sum of adjusted partial correlations in Corollary \ref{cor:measurement error bound} to obtain an upper bound. The result in Figure \ref{fig:upperbound} shows that this leads to a correct upper bound on the variance of the measurement error.
	
	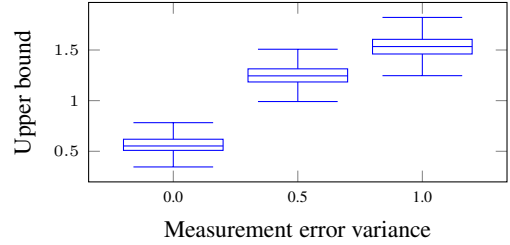
\begin{figure}[]
		\centering
		\begin{tikzpicture}
		\begin{axis}[
		width=0.9\linewidth,
		height=0.5\linewidth,
		boxplot/draw direction=y,
		xlabel = {\small Measurement error variance},
		ylabel = {\small Upper bound},
		xtick={1,2,3},
		xticklabels={0.0, 0.5, 1.0},
		boxplot/variable width,
		]
		\addplot [color=blue, boxplot prepared={draw position=1,
			lower whisker=0.3454, lower quartile=0.5089,
			median=0.5525, upper quartile=0.6190,
			upper whisker=0.7825},
		] coordinates {};
		\addplot [color=blue, boxplot prepared={draw position=2, 
			lower whisker=0.9911, lower quartile=1.1845,
			median=1.2447, upper quartile=1.3140,
			upper whisker=1.5075},
		] coordinates {};
		\addplot [color=blue, boxplot prepared={draw position=3,
			lower whisker=1.2466, lower quartile=1.4607,
			median=1.5340, upper quartile=1.6053,
			upper whisker=1.8217},
		] coordinates {};
		\end{axis}
		\end{tikzpicture}
		\caption{Simulation results for measurement error upper bound detection.}
		\label{fig:upperbound}
		\vspace{-8pt}
	\end{figure}
	
	\section{STRONG FAITHFULNESS}
	\label{sec:strong faithfulness}
	In this section we prove that conditional independences cannot be reliably detected in the presence of measurement error. We then discuss the \emph{strong} faithfulness assumption and its repercussions. In the next section we will present our error correction method, which relies on an upper bound for measurement error and the strong faithfulness assumption.
	
	Lemma \ref{lemma:independence faithfulness} shows that for two dependent (sets of) variables, a conditional independence between these variables can never be detected if the conditioning set is subject to measurement error, \emph{unless} the faithfulness assumption is violated.
	
	\begin{lemma}
		\label{lemma:independence faithfulness}
		Let $\B{X},\B{Y}$ and $\tilde{\B{Z}}$ be three sets of (disjoint) variables. If $\tilde{\B{Z}}$ has measurement error with non-zero variance, then the (in)dependences
		\begin{equation*}
		\B{X}\dep\B{Y} \qquad \B{X}\indep\B{Y}|\tilde{\B{Z}},
		\end{equation*}
		must be due to a violation of the faithfulness assumption.
	\end{lemma}
	\begin{proof}
		A faithfulness violation occurs when $\B{X}\indep\B{Y}|\tilde{\B{Z}}$ but $\tilde{\B{Z}}$ does not d-separate $\B{X}$ and $\B{Y}$. Since $\B{X}$ and $\B{Y}$ are dependent in the data there must be an open path between them by the Markov assumption. By definition of random measurement error the variables in $\tilde{\B{Z}}$ are leaf nodes. Therefore $\tilde{\B{Z}}$ cannot block the path between $\B{X}$ and $\B{Y}$, so that $\B{X}\dcon\B{Y}|\tilde{\B{Z}}$.
	\end{proof}

	Under the assumption that all variables in the model have the same measurement error variance (e.g. because they are subject to the same source of measurement error), the variance of the measurement error must be zero whenever a marginal dependence and a conditional independence is detected, as shown in Proposition \ref{prop:no measurement error}. 
	
	\begin{proposition}
		\label{prop:no measurement error}
		Let $\tilde{\B{X}},\tilde{\B{Y}}$ and $\tilde{\B{Z}}$ be three sets of (disjoint) variables with measurement errors that have equal (possibly zero) variances. Under the faithfulness assumption, if $\tilde{\B{X}}\dep\tilde{\B{Y}}$ and $\tilde{\B{X}}\indep\tilde{\B{Y}}|\tilde{\B{Z}}$, then the measurement error on all variables has zero variance.
	\end{proposition}
	\begin{proof}
		Follows directly from Lemma \ref{lemma:independence faithfulness}.
	\end{proof}
	
	Since constraint-based causal discovery algorithms rely both on the faithfulness assumption and on the results of conditional independence tests, poor performance is to be expected when variables are measured with error. In this article, we consider the \emph{strong faithfulness} assumption \citep{Spirtes2000} instead.
	
	\begin{assumptions}{(\emph{Strong faithfulness})}
		We assume that the data of the unobserved measurement-error-free variables is $\lambda$-strong faithful to the true underlying causal graph that generated it. That is, for all disjoint sets of variables $\B{X},\B{Y},\B{Z}$:
		\begin{equation*}
		|\rho_{\B{X},\B{Y}\given \B{Z}}| < \lambda \implies \B{X} \dsep \B{Y} \given \B{Z}.
		\end{equation*}
	\end{assumptions}
	
	The example in Figure \ref{fig:random measurement error} illustrates how the strong faitfhulness assumption may alleviate some of the negative effects of measurement error, but may aggravate the risk of detecting wrong conditional independences. If the data is $\lambda$-strong faitfhul, then it is also $\mu$-strong faithful, where $0<\mu\leq\lambda$, and $\mu$ can then be treated as a tuning parameter. In Figure \ref{fig:random measurement error}, for zero relative measurement error, the data is $\mu$-strong faithful for any $\mu$ up to $\lambda\sim 0.25$. For $\mu=0.25$ we find from the partial correlations that $X_1\indep \tilde{X}_3\given \tilde{X}_2$ upto a relative measurement error of approximately $0.3$, but for large enough measurement error we may also wrongly detect that $X_1\indep \tilde{X}_2\given X_3$.\footnote{Small enough correlations correspond to d-separations in the underlying graph by the strong faithfulness assumption. By the causal Markov assumption, d-separations correspond to conditional independences.}
		
	The tuning parameter thus represents a trade-off between detecting as many as possible of the true conditional independences and wrongfully detecting conditional independences. For the identification of LCD triples this means that for small $\mu$ and data that is corrupted by measurement error, we cannot detect the true LCD triples, while for large $\mu$ we may detect false LCD triples, because we detect conditional independences between variables that are actually dependent.	
	
	\section{ERROR PROPAGATION}
	\label{sec:error propagation}

	In this section we consider propagation of an error bound on random measurement error to partial correlations. If the strong faithfulness assumption holds, the effectiveness of tuning the threshold parameter $\lambda$ depends on the size of the measurement error. By taking measurement error into account, we aim to alleviate the adverse effect of wrongfully detecting conditional independences by including the possibility to adaptively assign `unkown' to a statistical test result. In that case we could get the best of both worlds: detect the correct conditional independences and assign `unknown' or `dependent' to the conditional dependences.		
	
	We start by defining an adjusted covariance matrix for three variables. Let $\boldsymbol{m}=(m_1,m_2,m_3)$ be the variances of the random measurement errors $(M_1,M_2,M_3)$ on the latent (unmeasured) variables $(X_1,X_2,X_3)$, and suppose that $\boldsymbol{u}^*=(u^*_1,u^*_2,u^*_3)$ is an upper bound such that $\boldsymbol{m} \preceq \boldsymbol{u}^*$.\footnote{$\preceq$ is the component-wise inequality between two vectors.} Suppose that $\tilde{\Sigma}$ is the true covariance matrix of the measured variables $\tilde{X}_1,\tilde{X}_2,\tilde{X}_3\in\boldsymbol{V}$. The adjusted covariance matrix is given by
	\begin{equation}
	\tilde{\Sigma} (\boldsymbol{u}) = \tilde{\Sigma} - \boldsymbol{u}^T I,
	\end{equation}
	where $I$ denotes the identity matrix, when $\tilde{\Sigma}(\boldsymbol{u})$ has an inverse, otherwise $\tilde{\Sigma}(\boldsymbol{u})=\tilde{\Sigma}$.
	
 	For $0\preceq \boldsymbol{u} \preceq \boldsymbol{u}^*$ we can find minimal and maximal absolute values of partial correlations based on $\tilde{\Lambda}(\boldsymbol{u})=(\tilde{\Sigma}(\boldsymbol{u}))^{-1}$. We define
	\begin{align}
	\tilde{\rho}^{\text{min}}_{12|3} &= \argmin_{0\preceq \boldsymbol{u} \preceq \boldsymbol{u}^*} 
	\left\rvert\frac{(\tilde{\Lambda}(\boldsymbol{u}))_{12}}{\sqrt{(\tilde{\Lambda}(\boldsymbol{u}))_{11} (\tilde{\Lambda}(\boldsymbol{u}))_{22}}}\right\rvert,\\
	\tilde{\rho}^{\text{max}}_{12|3} &= \argmax_{0\preceq \boldsymbol{u} \preceq \boldsymbol{u}^*} \left\rvert\frac{(\tilde{\Lambda}(\boldsymbol{u}))_{12}}{\sqrt{(\tilde{\Lambda}(\boldsymbol{u}))_{11} (\tilde{\Lambda}(\boldsymbol{u}))_{22}}}\right\rvert.
	\end{align}
	
	Under the $\lambda$-strong faithfulness assumption, the conditional (in)dependence relations can be determined as follows:
	\begin{equation}
	\label{eq:propagated test}
	\begin{cases}
	X_1\dep X_2 \given X_3 & \,\text{ if }  \tilde{\rho}^{\text{min}}_{12|3} > \lambda\\
	X_1\indep X_2\given X_3 & \,\text{ if } \tilde{\rho}^{\text{max}}_{12|3} < \lambda,
	\end{cases}.
	\end{equation}
	The nature of the relation is undecided when $\tilde{\rho}^{\text{min}}_{12|3} <\lambda$ and $\tilde{\rho}^{\text{max}}_{12|3}>\lambda$.\footnote{In practical applications the covariance matrix $\tilde{\Sigma}$ is estimated from data. The added uncertainty can be taken into account by using bootstrapping to obtain confidence intervals for $\tilde{\rho}^{\text{min}}_{12|3}$ and $\tilde{\rho}^{\text{max}}_{12|3}$.}
	
	Although we consider a measurement error correction in cases where only one variable is conditioned upon, our ideas can be trivially extended to accommodate larger conditioning sets when an upper bound on the measurement error is known for all variables involved\footnote{In that case one considers a larger adjusted covariance matrix, and since the partial correlations are calculated from the covariance matrix one can use the same scheme to find minimal and maximal values for the absolute partial correlation.}.
	
	\section{DATA SIMULATIONS}
	We now evaluate the effects of a measurement error correction on simulated data. For detailed descriptions of the simulation settings, we refer to the supplementary material.
	
	     \begin{figure*}[]
		\begin{subfigure}[b]{0.44\textwidth}
			\begin{tikzpicture}
			\begin{axis}[
			width=\linewidth,
			height=0.58\linewidth,
			xlabel={\small Measurement Error},
			xmin=0,xmax=1.0,xtick={0,0.2,...,1.0},
			ylabel={\small Error Rate},
			yticklabel style={/pgf/number format/fixed, /pgf/number format/precision=2, /pgf/number format/fixed zerofill},
			ymin=0.0,ymax=0.04,
			ytick={0.0,0.01,0.02,0.03,0.04},
			scaled y ticks=false,
			ylabel shift=-2pt,
			]
			\addplot[blue!90!white, line width=1.2pt] table[x index = {0}, y index = {1}, col sep = comma] {CI_dep.csv};
			\addplot[mycolor3, dotted, line width=1.5pt] table[x index = {0}, y index = {2}, col sep = comma] {CI_dep.csv};
			\addplot[applegreen, line width=1.2pt] table[x index = {0}, y index = {4}, col sep = comma] {CI_dep.csv};
			\addplot[mycolor2, dotted, line width=1.5pt] table[x index = {0}, y index = {3}, col sep = comma] {CI_dep.csv};
			\end{axis}
			\end{tikzpicture}
			\caption{Conditional dependences.}
			\label{fig:cd}
		\end{subfigure}
		\begin{subfigure}[b]{0.44\textwidth}
			\begin{tikzpicture}
			\begin{axis}[
			clip mode=individual,
			width=\linewidth,
			height=0.58\linewidth,
			xlabel={\small Measurement Error},
			xmin=0,xmax=1.0,xtick={0,0.2,0.4,...,1.0},
			ylabel={\small Error Rate},
			ymin=-0.05,ymax=1.0,ytick={0,0.2,...,1.0},
			ylabel shift=-2pt,
			axis on top,
			legend pos=outer north east,
			legend columns=1,
			legend cell align={left}
			]
			\addplot[applegreen, line width=1.2pt] table[x index = {0}, y index = {4}, col sep = comma] {CI_indep.csv};
			\addplot[mycolor2, dotted, line width=1.5pt] table[x index = {0}, y index = {3}, col sep = comma] {CI_indep.csv};
			\addplot[blue!90!white, line width=1.2pt] table[x index = {0}, y index = {1}, col sep = comma] {CI_indep.csv};
			\addplot[mycolor3, dotted, line width=1.5pt] table[x index = {0}, y index = {2}, col sep = comma] {CI_indep.csv};
			\addplot[mycolor1, line width=1.2pt] coordinates {
				(0, -1)
				(1, -1)
			};
			\addplot[dashed, black, line width=1.2pt] coordinates {
				(0, -1)
				(1, -1)
			};
			\addlegendimage{applegreen, line width=1.2pt}                
			\addlegendimage{mycolor2, dotted, line width=1.5pt}
			\addlegendimage{blue!90!white, dotted, line width=1.2pt}
			\addlegendimage{mycolor3, line width=1.5pt}
			\addlegendimage{mycolor1, line width=1.2pt}
			\addlegendimage{line width=1.2pt}
			\legend{$\alpha$, $\lambda$, $t=1.0$, $t=1.5$, $t$, baseline,};
			\end{axis}
			\end{tikzpicture}
			\caption{Conditional independences.}
			\label{fig:ci}
		\end{subfigure}
		\newline
		\begin{subfigure}[b]{0.44\textwidth}
			\begin{tikzpicture}
			\begin{axis}[
			width=\linewidth,
			height=0.58\linewidth,
			xlabel={\small Recall},
			xmin=0.0, xmax=1.01, xtick={0,0.2,0.4,...,1.0},
			ylabel={\small Precision},
			ymin=0.45, ymax=1.01, try min ticks=6,
			ylabel shift=-2pt,
			]
			\addplot[dashed, black, line width=1.2pt] coordinates {
				(0, 0.5)
				(1, 0.5)
			};
			\addplot[blue!90!white, line width=1.2pt] table[x index = {0}, y index = {2}, col sep = comma] {LCD_pr_ub.csv};
			\addplot[applegreen, line width=1.2pt] table[x index = {1}, y index = {3}, col sep = comma] {LCD_pr_ub.csv};
			\addplot[mycolor2, dotted, line width=1.5pt] table[x index = {0}, y index = {1}, col sep = comma] {LCD_pr_lambda.csv};
			\addplot[mycolor3, dotted, line width=1.5pt] table[x index = {0}, y index = {1}, col sep = comma] {LCD_pr_alpha.csv};        
			\end{axis}
			\end{tikzpicture}
			\caption{Application to LCD.}
			\label{fig:pr}
		\end{subfigure}
		\begin{subfigure}[b]{0.44\textwidth}
			\begin{tikzpicture}
			\begin{axis}[
			width=\linewidth,
			height=0.58\linewidth,
			xlabel={\small Recall},
			xticklabel style={/pgf/number format/fixed, /pgf/number format/precision=2, /pgf/number format/fixed zerofill},
			scaled x ticks=false,
			xmin=0.0,xmax=0.3,xtick={0,0.05,0.1,...,0.4},
			ylabel={\small Precision},
			ymin=0.0,ymax=0.4,ytick={0,0.1,...,1.0},
			ylabel shift=-2pt,
			]
			
			\addplot[dashed, black, line width=1.2pt] coordinates {
				(0, 0.016)
				(1, 0.016)
			};
			\addplot[mycolor1, line width=1.2pt] table[x index = {0}, y index = {1}, col sep = comma] {pr_ub_comb.csv};
			\addplot[blue!90!white, line width=1.2pt] table[x index = {0}, y index = {1}, col sep = comma] {pr_ub_comb2.csv};
			\addplot[mycolor3, dotted, line width=1.5pt] table[x index = {0}, y index = {1}, col sep = comma] {pr_ub_comb3.csv};
			\addplot[applegreen, line width=1.2pt] table[x index = {0}, y index = {1}, col sep = comma] {pr_alpha_comb.csv};
			\addplot[mycolor2, dotted, line width=1.5pt] table[x index = {0}, y index = {1}, col sep = comma] {pr_lambda_comb.csv};      
			\end{axis}
			\end{tikzpicture}
			\caption{Upper bound detection and application to LCD.}
			\label{fig:cpr}
		\end{subfigure}
		\caption{Simulation results. Figures \subref{fig:cd} and \subref{fig:ci} show the error rate for detecting conditional dependences and independences in the presence of measurement error for $\lambda$-strong faithful data. It is assumed that $\lambda=0.1$ is known, and $\alpha=0.05$. Figure \subref{fig:pr} shows the precision-recall curve for detecting LCD triples from $\lambda$-strong faithful data subject to measurement error with fixed variance and a given upper bound, where $\lambda$ and $\alpha$ are used as tuning parameters. Figure \subref{fig:cpr} shows the precision-recall curve for simulations of $15$ variables, where we first apply the upper bound detection and then the measurement error correction and $\alpha$ and $\lambda$ are treated as tuning parameters. The baseline is at $0.016$.}
		\vspace*{-4pt}
	\end{figure*}
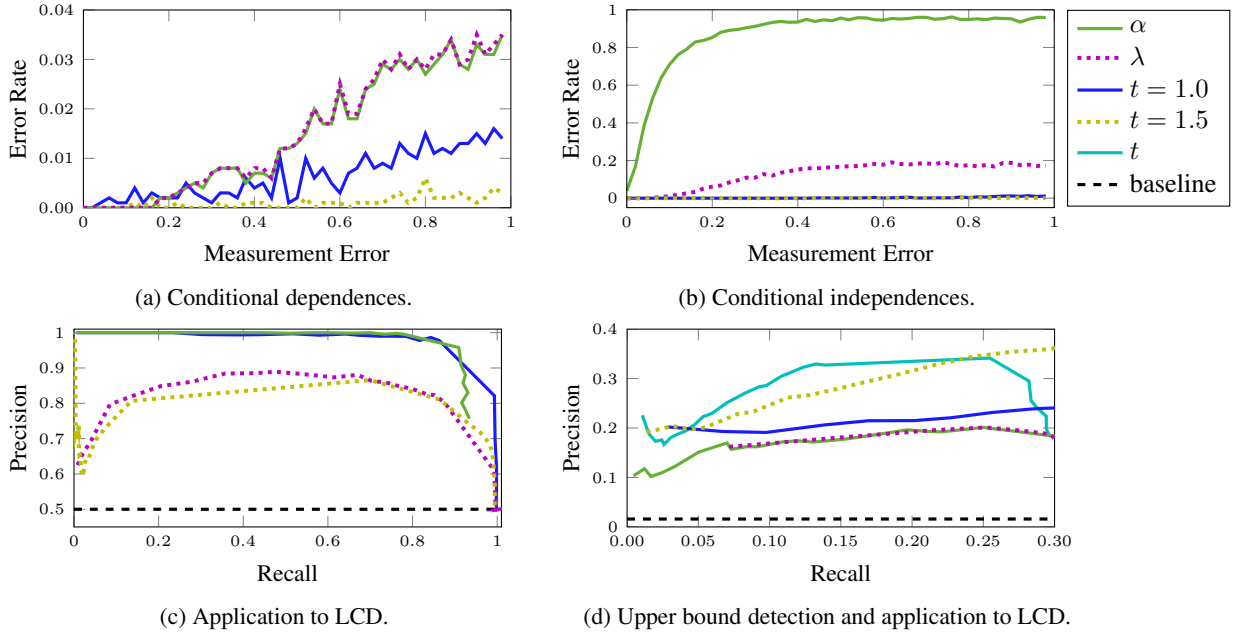

	\subsection{CONDITIONAL INDEPENDENCE TESTING} To illustrate the effectiveness of the measurement error correction in identifying conditional (in)dependence relations, we generated data for three variables $(X_1,X_2,X_3)$ from linear-Gaussian acyclic causal structures, possibly with latent confounders.
	We only considered triples that satisfied the $\lambda$-strong faithfulness assumption for $\lambda=0.1$ and $X_1\dep X_2$ and $X_2\dep X_3$. 
	
	We simulated $2000$ models where half of the models satisfied $X_1\indep X_3\given X_2$. From each model we generated $10000$ samples and added normally distributed random measurement error to each variable with varying variance. The conditional (in)dependence between $\tilde{X}_1\indep \tilde{X}_3\given \tilde{X}_2$ was tested in various ways: using a threshold on the p-value $\alpha=0.05$, using a threshold $\lambda=0.1$ on the partial correlation, and using the same threshold with a measurement error correction with an upper bound on the measurement error of $t$ times the true variance. We then calculated the error rate as the number of incorrect classifications relative to the total number of tests. Note that the amount of conditional (in)dependence relations that are assigned `unknown' increases with the size of the measurement error and the tightness of the upper bound. For an evaluation of the amount of `unknown' classifications we refer to the supplementary material.
	
	Figure \ref{fig:cd} shows that the measurement error correction slightly reduces the error rate for conditional dependences, and \ref{fig:ci} shows that the error rate of detecting incorrect conditional dependences is greatly reduced. 
	
	\subsection{APPLICATION TO LCD}
	Typically, when data is $\lambda$-strong faithful to the true underlying causal graph, the value of $\lambda$ is not known and $\lambda$ is therefore used as a tuning parameter instead. We generated triples $(X_1,X_2,X_3)$ as in the previous simulation, but only selected triples where $X_1$ was not caused by $X_2$ and $X_3$. We added measurement error with a fixed variance. We then applied the LCD algorithm, testing conditional independences as in the previous section. We evaluated the results by checking whether the causal structure of the triples was correctly identified. Figure \ref{fig:pr} shows that the precision of the algorithm with the measurement error corrected test results outperforms the standard methods.
	
	We also consider the more realistic case where multiple variables are measured, the upper bound for the variance of the measurement error is not known in advance, and the data is not necessarily $\lambda$-strong faithful. To that end we simulated $10000$ datapoints from a random acyclic model with $15$ variables, where one variable was not caused by any of the other variables.
	We added measurement error to each variable with a fixed variance.
	
	For the upper bound detection, we first tested whether the tetrad constraints vanished using Wishart's test \citep{Wishart1928} at the $5\%$ level, and then used the result in Corollary \ref{cor:measurement error bound} to obtain an upper bound for the measurement error. When we found multiple upper bounds (for multiple variables) we chose the median as an upper bound for the measurement error on all variables. Finally we applied the LCD algorithm, testing marginal (in)dependences with a t-test at the $5\%$ level and conditional (in)dependences as in the previous experiments and using the detected upper bound. If we were not able to detect an upper bound, we assigned `unknown' to every test result. We checked how often a correct causal structure was identified. In $200$ repetitions of the experiment the upper bound was incorrect in only $3$ cases and no upper bound was detected in $39$ cases. Figure \ref{fig:cpr} shows that all methods score significantly better than the random baseline and that the precision for detecting LCD triples increases significantly when we use the measurement error correction.
			
	\section{PROTEIN SIGNALING NETWORKS}
	We present an application of our ideas to real-world protein signaling data that could be corrupted by measurement error. We used a dataset concerning the influence of protein abundances on the properties of a protein signaling network in human kidney cells \citep{Lun2017}, and obtained an upper bound for the variance of random measurement error from this data. In absence of a reliable ground truth for this experiment, we validated the results of a measurement error correction applied to the LCD algorithm by comparing it to a baseline derived from interventions in the data.
	
	\paragraph{Data description}
	For conditions $j = 1,\ldots,20$ the abundance of a different protein labeled $(\text{GFP})_j$ was over-expressed and then measured \citep{Lun2017}. The abundances of an additional $34$ \emph{phosphorylated} proteins $P_i$ were measured after stimulation of the network. We relabeled conditions $j$ so that over-expression of a protein $(\text{GFP})_j$ corresponds to the measured phosphorylated abundance $P_j$.
	
	The abundance of an over-expressed protein typically differed between cells and not every cell was affected \citep{Lun2017}. Because of the experimental design, $(\text{GFP})_j$ is not caused by the abundance or phosphorylation of the other proteins, which allowed us to treat the abundance of $(\text{GFP})_j$ as an intervention variable.
	
	Typically $\sim 10000$ single cells were measured for each condition. We assume that the data-generating process can be approximated by a linear-Gaussian model after pre-processing. For details about data pre-processing we refer to the supplementary material.
	
	\paragraph{Upper bound detection}
	We considered all proteins under over-expression of the SRC protein, for which strong signaling relations were present (see also \citep{Lun2017}). For all $4$-tuples $((\text{GFP})_{\text{SRC}},P_i,P_j,P_k)$ that were all marginally dependent at the $1\%$ level (using a t-test), we tested whether all three tetrads vanished using Wishart's test at the $5\%$ level. We found that these constraints were satisfied for the $4$-tuple ($(\text{GFP})_{\text{SRC}}$, pS6K, pMAPKAPK2, pMAP2K3).
	
	This allowed us to apply the results presented in Section \ref{sec:upperbound} to obtain an upper bound. The upper bounds for the variance of measurement error that we found were $0.10$ for pS6K, $0.15$ for pMAPKAPK2, and $0.14$ for pMAP2K3.\footnote{Each of the detected other bounds corresponds to adjusting the corresponding variable, as in Corollary \ref{cor:measurement error bound}. Other triples that satisfied the constraints gave similar or (much) higher upper bounds for the measurement error.} Since all proteins were measured with the same device, we assumed that the variance of the measurement error is the same for each variable, so that $0.14$ is a suitable upper bound for the measurement error on any variable. 
	
	Although the detected upper bound was large for weak signals, the proteins with stronger signals typically had variances $>1$, so that the relative amount of measurement error for proteins with strong signaling relations amounted to less than $10\%$.
	
	\paragraph{Baseline}	
	To validate the results of LCD, we created a baseline from the interventions (corresponding to over-expression of certain proteins) in the dataset. A reasonable assumption is that  $(\text{GFP})_j$ is a direct cause of $P_j$, because the higher the abundance of a protein, the more it can be phosphorylated. Under the assumption that over-expression of a protein $P_j$ does not alter the network structure \citep{Lun2017} and that $(\text{GFP})_j$ does not directly cause any of the other proteins $P_i$, with $i \neq j$, we have that $P_j$ is a cause of $P_k$, whenever $(\text{GFP})_j$ and $P_k$ are dependent. 
	
	We constructed a baseline for cause-effect pairs $(P_j,P_k)$, where we considered $7$ phosphorylated proteins $P_j$ that were over-expressed in one of the conditions as cause variables and all $34$ phosphorylated proteins as effect variables. The subset of proteins that was used to construct the baseline follows the recommendations in \citet{Lun2017}. We considered a pair $(P_j, P_k)$ a causal pair, if a t-test indicated that $(\text{GFP})_j$ and $P_k$ were dependent at a level of $10^{-4}$. This resulted in $231$ possible cause-effect pairs, $71\%$ of which were cause-effect pairs in the baseline.
	
	\paragraph{Methods and results}
	We applied the LCD algorithm to the data to identify causal pairs $(P_j,P_k)$ by treating $(\text{GFP})_i$ as an intervention variable for conditions $i\in\{1,\ldots,20\}$, with $i\neq j$ and $i\neq k$. Since central proteins in the network were over-expressed, true causal pairs were expected to appear under multiple conditions. To make our results more robust, we only made a positive prediction if a causal pair was predicted for at least $2$ conditions. We applied three methods of conditional (in)dependence testing in combination with the LCD algorithm: a threshold $\alpha$ on the p-value of t-tests, a threshold $\lambda$ on the absolute value of partial correlations, and a threshold $\lambda$ on partial correlations with a measurement error correction using the upper bound $u=0.14$.	
	
	By treating $\lambda$ and $\alpha$ as tuning parameters and taking the baseline as ground truth, we calculated a precision-recall curve for each method of conditional (in)dependence testing. The results are displayed in Figure \ref{fig:lcd results}, which shows that $\alpha$ and $\lambda-u$ have comparable pr-curves. For this dataset there seems to exist a threshold $\alpha$ that is already able to distinguish between conditional independences and dependences. We can see that both $\lambda-u$ and $\alpha$ significantly outperform random guessing, but $\lambda$ does not. Although the differences between the methods are not significant, this seems to indicate that a measurement-error correction improves the precision at low recall for conditional independence testing with a fixed threshold on (partial) correlations.
	
		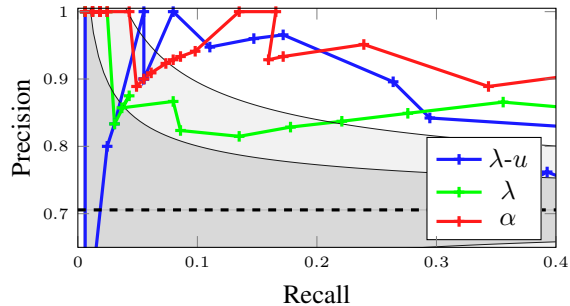
\begin{figure}[t]
			\centering
			\begin{tikzpicture}
			\begin{axis}[
			width=\linewidth,
			height=0.6\linewidth,
			xlabel={Recall},
			xmin=0,xmax=0.4,xtick={0,0.1,...,1},
			ylabel={Precision},
			ymin=0.65,ymax=1.005,ytick={0.6,0.7,...,1},
			ylabel shift=-2pt,
			legend pos=south east,
			axis on top,
			]
			
			\addplot[black!25!white, line width=0.5pt, forget plot] coordinates {
				(0, 0.7056277)
				(1, 0.7056277)
			};	
			
			\addplot[blue!90!white, mark=+, line width=1.2pt] table[x index = {0}, y index = {1}, col sep = comma] {pr_lambda.csv};
			\addplot[green, mark=+, line width=1.2pt] table[x index = {0}, y index = {1}, col sep = comma] {pr_lambda_nc.csv};
			\addplot[red!90!white, mark=+, line width=1.2pt] table[x index = {0}, y index = {1}, col sep = comma] {pr_alpha.csv};
			
			\addplot[black!90!white, line width=0.3pt, forget plot, name path=C] table[x index = {0}, y index = {2}, col sep = comma] {ci_baseline.csv};
			\addplot[black!90!white, line width=0.3pt, forget plot, name path=D] table[x index = {0}, y index = {4}, col sep = comma] {ci_baseline.csv};
			
			\addplot[black!5, forget plot] fill between[of=C and D];
			
			\addplot[black!90!white, line width=0.3pt, forget plot, name path=A] table[x index = {0}, y index = {1}, col sep = comma] {ci_baseline.csv};
			\addplot[black!90!white, line width=0.3pt, forget plot, name path=B 	] table[x index = {0}, y index = {3}, col sep = comma] {ci_baseline.csv};
			
			\addplot[black!15, forget plot] fill between[of=A and B];
			
			\addplot[dashed, black, line width=1.2pt] coordinates {
				(0, 0.7056277)
				(1, 0.7056277)
			};
			
			\addlegendentry{$\lambda$-$u$};
			\addlegendentry{$\lambda$};
			\addlegendentry{$\alpha$};
			\end{axis}
			\end{tikzpicture}
			\caption{LCD applied to protein signaling data with $\alpha$ or $\lambda$ as tuning parameter and a measurement-error correction. The results are compared with the random baseline, the gray-shaded areas represent one and two standard deviations from the random baseline.}
			\label{fig:lcd results}
		\end{figure}	
		
	\section{CONCLUSION}
	
	In this paper we demonstrated that measurement error, when not taken into account, can fool causal discovery methods into wrongfully inserting, deleting or reversing edges in the predicted causal graph. We showed that regular statistical tests with conventional thresholds would fail to detect conditional independences between the uncorrupted variables from the measurement data when measurement error is present. We also proposed a correction method aimed at mitigating the negative effects of measurement error.
	
	The key result that we presented in this work is that, under certain conditions, we can find an upper bound for the variance of the measurement error from data that has been corrupted by measurement error. We show that this uncertainty can be propagated into an uncertainty regarding the partial correlations to correct for measurement error. We showed a successful application of our approach on simulated data.
	
	We also applied our ideas to a real-world protein signaling dataset, and we found an upper bound for the variance of the measurement error in this dataset. We found that our approach gave significantly higher precision than a random baseline. However, also the conventional method without correction for measurement error seems to work well on this dataset. Nevertheless, it is our belief that taking measurement error into account is a promising step towards successful real-world applications of (constraint-based) causal discovery.	
	
	\subsubsection*{Acknowledgements}
	This work was supported by the European Research Council (ERC) under the European Union's Horizon 2020 research and innovation programme (grant agreement 639466). We thank Ioannis Tsamardinos, Sofia Triantafillou and Karen Sachs for providing useful feedback on initial drafts of this work.
	                                                       	
	\clearpage
	\bibliography{library}
	
	\clearpage
	\onecolumn\appendix
	
	\section*{SUPPLEMENT}
	
	\section{Proof of Theorem 1}
	Throughout the proof, we denote the covariance between variables $\tilde{X}_i$ and $\tilde{X}_j$ as $\tilde{\Sigma}_{ij}$ for $i,j\in\{1,2,3,4\}$. First we prove the direction `$\impliedby$':
	\begin{align*}
	&\begin{cases}
	\tilde{\rho}^u_{13|2} = 0 \iff \tilde{\Sigma}_{12}\tilde{\Sigma}_{23} - \tilde{\Sigma}_{13}(\tilde{\Sigma}_{22}-u)=0\\
	\tilde{\rho}^u_{14|2} = 0 \iff \tilde{\Sigma}_{12}\tilde{\Sigma}_{24} - \tilde{\Sigma}_{14}(\tilde{\Sigma}_{22}-u)=0\\
	\tilde{\rho}^u_{34|2} = 0 \iff \tilde{\Sigma}_{32}\tilde{\Sigma}_{24} - \tilde{\Sigma}_{34}(\tilde{\Sigma}_{22}-u)=0
	\end{cases}\\
	&\iff (\tilde{\Sigma}_{22}-u) = \frac{\tilde{\Sigma}_{12}\tilde{\Sigma}_{23}}{\tilde{\Sigma}_{13}} =  \frac{\tilde{\Sigma}_{12}\tilde{\Sigma}_{24}}{\tilde{\Sigma}_{14}} =\frac{\tilde{\Sigma}_{23}\tilde{\Sigma}_{24}}{\tilde{\Sigma}_{34}}\\
	&\iff \frac{\tilde{\Sigma}_{23}}{\tilde{\Sigma}_{13}} = \frac{\tilde{\Sigma}_{24}}{\tilde{\Sigma}_{14}},\qquad \frac{\tilde{\Sigma}_{12}}{\tilde{\Sigma}_{14}} = \frac{\tilde{\Sigma}_{23}}{\tilde{\Sigma}_{34}}\\
	&\iff \tilde{\Sigma}_{12}\tilde{\Sigma}_{34}=\tilde{\Sigma}_{13}\tilde{\Sigma}_{24}=\tilde{\Sigma}_{14}\tilde{\Sigma}_{23}.
	\end{align*}
	The result follows by applying Lemma 1 in the main paper, and observing that these are the only structures for a random measurement model where all d-separations hold.
	
	Now we prove `$\implies$'. Since the true underlying causal graph is as in Figure 1 in the main paper, we have that there is an $\alpha$ such that $\tilde{X}_2 = \alpha L + E_2 + M_2$ for some $\alpha\neq 0$, where $E_2$ is an independent noise variable with variance $\tau$ for $X_2$ and $M_2$ is an independent random measurement error for $\tilde{X}_2$. We let $\Sigma_{ij}$ denote the covariance between variables $X_i$ and $X_j$ for $i,j\in\{1,2,3,4\}$. Covariances between $L$ and variables $(X_1,X_2,X_3,X_4)$ are denoted as $\Sigma_{Li}$ and $\Sigma_{iL}$ for $i\in\{1,2,3,4\}$. Hence
	\begin{equation*}
	\text{Cov}(\tilde{X}_1,\tilde{X}_2,\tilde{X}_3) =
	\begin{pmatrix}
	\Sigma_{11} + m_1 & \alpha \sigma_{L1} & \Sigma_{13}\\
	\alpha \Sigma_{L1} & \alpha^2 \Sigma_{LL} + \tau + m_2 & \alpha \Sigma_{L3}\\
	\Sigma_{13} & \alpha \Sigma_{L3} & \Sigma_{33} + m_3,
	\end{pmatrix}.
	\end{equation*}
	where $m_1,m_2,m_3$ are the variances of $M_1,M_2,M_3$ respectively and $\Sigma_{LL}$ denotes the variance of the latent variable $L$. From this we obtain the relation for the adjusted partial correlation:
	\begin{align*}
	\label{eq:partial correlation x2}
	\tilde{\rho}^u_{13|2} = 0 \qquad
	\iff  \qquad \alpha^2(\Sigma_{L1}\Sigma_{L3} - \Sigma_{13}\Sigma_{LL}) - \Sigma_{13}(\tau + m_2-u) = 0.
	\end{align*} 
	Since $L$ d-separates $X_1$ and $X_3$, the partial correlation $\tilde{\rho}_{13|L}=0$ by the Markov assumption. Therefore
	\begin{equation*}
	\Sigma_{L1}\Sigma_{L3} - \Sigma_{13}\Sigma_{LL} = 0.
	\end{equation*}
	Because $\Sigma_{13}\neq 0$ by assumption, we find that $\tilde{\rho}^u_{13|2}=0$ if and only if $u = \tau + m_2$. Via a similar argument we can show that for this $u$ we also have that $\tilde{\rho}^u_{14|2}=0$ and $\tilde{\rho}^u_{34|2}=0$.
	
	\section{Data simulations}
	In this section we give some additional details about the simulations that we used for the experiments in Section 7 of the main paper.
	
	We obtained the results in Figure 7a of the main paper, by generating random DAGs for $6$ variables with a connection probability of $0.7$, parameters chosen uniformly at random from the interval $[-1.0,1.0]$, and error variances chosen uniformly from the interval $[0.5,1.0]$. Three out of the $6$ variables were observed variables, and the remaining three were latent. We then used rejection sampling to select models for which the observed variables $(X_1,X_2,X_3)$ satisfied: the $\lambda$-strong faithfulness assumption for $\lambda=0.1$, $X_1\dep X_2$, $X_2\dep X_3$, and $X_1\dep X_3\given X_2$. For each model, we generated $10000$ datapoints and added measurement error with varying variances. 
	
	For the experiment in Figure 7b in the main paper, we generated models for three variables $(X_1,X_2,X_3)$ that satisfy: $X_1\dep X_2$, $X_2\dep X_3$, and $X_1\indep X_3\given X_2$. To do this, we considered all causal structures for $(X_1,X_2,X_3)$ that satisfied these conditions. For each causal structure, parameters were chosen uniformly at random from the interval $[-1.0,1.0]$ and error variances were chosen uniformly at random from the interval $[0.5,1.0]$. We then used rejection sampling to select model that satisfied the $\lambda$-strong faithfulness for $\lambda=0.1$. For each model, we generated $10000$ datapoints and added measurement error with varying variances. We made sure that an equal amount of models was selected from each causal structure.
	
	For the experiments in Figures 7a and 7b in the main paper, the measurement-error corrected method also gives the output `unknown'. The rate of conditional dependences and independences that could not be detected in the experiment are shown in Figures \ref{fig:cd uk} and \ref{fig:ci uk} respectively.
	
	To obtain the results in Figure 7c, we generated causal structures for triples of variables $(X_1,X_2,X_3)$. To generate data for triples that did not have the structure of an LCD triple, we used the simulations for Figure 7a, but only selected causal structures where $X_1$ was not caused by $X_2$ and $X_3$. Similarly, to generate data for triples that did have the structure of LCD triples, we used the simulations for Figure 7b, but only selected causal structures where $X_1$ can be treated as an intervention variable. We added measurement error with a fixed variance of $0.8$.
	
	Finally, in order to obtain the results in Figure 7d, we generated random DAGs with $15$ nodes and a connection probability of $0.15$. In this case edge weights were chosen uniformly at random from the interval $[0.8,1.2]$ and error variance were chosen uniformly at random from the interval $[0.5,1.0]$. After generating $10000$ datapoints for each model, we added measurement error with a fixed variance of $0.8$.
	
	\definecolor{mycolor1}{rgb}{0.00000,0.75000,0.75000}
	\definecolor{mycolor2}{rgb}{0.75000,0.00000,0.75000}
	\definecolor{mycolor3}{rgb}{0.75000,0.75000,0.00000}
	\definecolor{ao(english)}{rgb}{0.11, 0.35, 0.02}
	\definecolor{applegreen}{rgb}{0.4, 0.69, 0.2}
	
	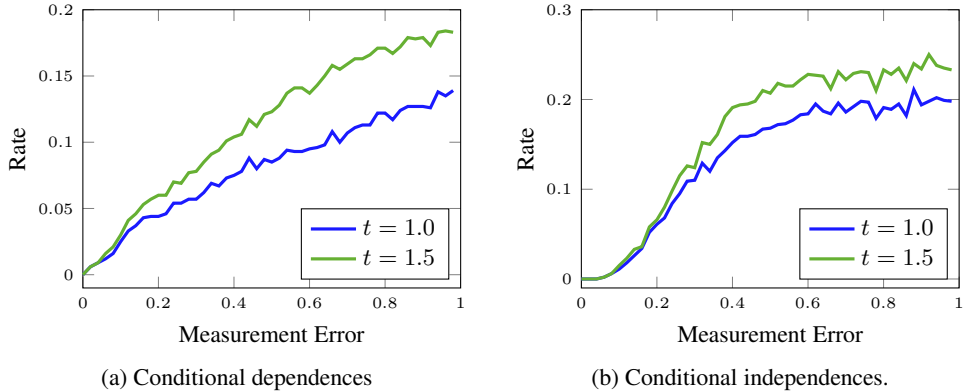
\begin{figure}[]
		\centering
		\begin{subfigure}[]{0.4\textwidth}
			\centering
			\begin{tikzpicture}
			\begin{axis}[
			width=\linewidth,
			height=0.8\linewidth,
			xlabel={\small Measurement Error},
			xmin=0,xmax=1.0,xtick={0,0.2,...,1.0},
			ylabel={\small Rate},
			yticklabel style={/pgf/number format/fixed},
			ymin=-0.01,ymax=0.2,ytick={0.0,0.05,0.10,0.15,0.2},
			ylabel shift=-2pt,
			legend pos=south east,
			]		
			\addplot[blue!90!white, line width=1.2pt] table[x index = {0}, y index = {5}, col sep = comma] {CI_dep.csv};
			\addplot[applegreen, line width=1.2pt] table[x index = {0}, y index = {6}, col sep = comma] {CI_dep.csv};	
			
			\addlegendentry{\small $t=1.0$};
			\addlegendentry{\small $t=1.5$};	
			\end{axis}
			\end{tikzpicture}
			\caption{Conditional dependences}
			\label{fig:cd uk}
		\end{subfigure}
		\begin{subfigure}[]{0.4\textwidth}
			\centering
			\begin{tikzpicture}
			\begin{axis}[
			width=\linewidth,
			height=0.8\linewidth,
			xlabel={\small Measurement Error},
			xmin=0,xmax=1.0,xtick={0,0.2,0.4,...,1.0},
			ylabel={\small Rate},
			ymin=-0.01,ymax=0.3,ytick={0,0.1,0.2,0.3},
			ylabel shift=-2pt,
			legend pos=south east,
			]	
			
			\addplot[blue!90!white, line width=1.2pt] table[x index = {0}, y index = {5}, col sep = comma] {CI_indep.csv};
			\addplot[applegreen, line width=1.2pt] table[x index = {0}, y index = {6}, col sep = comma] {CI_indep.csv};
			
			\addlegendentry{\small $t=1.0$};
			\addlegendentry{\small $t=1.5$};	
			\end{axis}
			\end{tikzpicture}
			\caption{Conditional independences.}
			\label{fig:ci uk}
		\end{subfigure}
		\caption{Rate of conditional dependences and independences that were not detected using a measurement error correction in the experiment in Figures 7a and 7b in the main paper.}
	\end{figure}
	
	\section{Protein Signaling Data}
	
	The raw data was pre-processed by transforming each datapoint $x$ by
	\begin{equation*}
	\hat{x} = \text{arcsinh}(x/5).
	\end{equation*}
	As a preprocessing step we filtered out cells that are in the M cell cycle phase according to the gating procedure described in \citep{Behbehani2012}. We motivate this filtering step by several reasons. First, cells in the M phase form a distinct cluster and strongly violate our assumption of a linear-Gaussian model since they represent a separate cluster. Second, in the M phase the cells already have doubled nuclei and some other organelle, so we cannot safely assume that the causal mechanisms of signalling are the same anymore. Therefore, the removal of these cells should be seen as removal of a contaminating population. Practically this came down to selecting only single cell measurements for which the abundance of the phosphorylated protein pHH3 was smaller than $3.0$.
	
	We then only included data in our (conditional) independence tests, when all measurements that were needed to conduct the test exceeded a lower threshold of $0.5$, to account for the detection limit in mass cytometry. 
	
	For the analysis in the main paper, we only considered the proteins that were over-expressed and whose phosphorylated abundances were also measured. We considered the measurements $5$ minutes after stimulation, because at this time-point the signaling responses were generally strong, see also Figure 3 in \citep{Lun2017}. For our analysis we only used the first replica of the experiment, which had the most measurements for each condition. 
	
	We analyzed a subset of the available proteins, based on the recommendations in \citet{Lun2017}, and excluded proteins from the cause variables when spill-over effects were reported under the condition that they were over-expressed, see also Table \ref{tab:cause proteins}. We also excluded pS6 because over-expressing it induced no strong signaling responses. Finally, we also discarded SHP2. Although the condition where SHP2 was over-expressed was not affected by spill-over effects, the measured \emph{phosphorylated} abundances of pSHP2 were affected by spill-over affects under multiple conditions.
	
	\begin{table}[h!]
		\centering
		\caption{Proteins that are both over-expressed in one of the conditions and whose phosphorylated abundance is measured under all conditions, with an indication whether spillover effects are present.}
		\label{tab:cause proteins}
		\begin{tabular}{l l l}
			Over-expressed protein & Measured protein & Spill-over\\
			\midrule
			JNK1 & pJNK & no \\
			MKK6 & pMKK3/6 & no \\
			PDPK1 & pPDPK1 & yes \\
			P38 & pP38 & no \\
			AKT1 & pAKT & no \\
			ERK2 & pERK & no \\
			SHP2 & pSHP2 & no \\
			GSK3B & pGSK3B & yes \\
			S6 & pS6 & no*\\
			P90RSK & pP90RSK & yes \\
			MEK1 & pMEK1/2 & no \\
			P70S6K & pS6K & no\\
			\bottomrule
		\end{tabular}
	\end{table}
	
\end{document}